\documentclass{article}

\usepackage[preprint]{neurips_2025}

\usepackage{amsmath,amssymb,amsfonts,amsthm}
\usepackage{booktabs}
\usepackage{graphicx}
\newtheorem{proposition}{Proposition}
\newtheorem{lemma}{Lemma}
\newtheorem{corollary}{Corollary}
\newtheorem{remark}{Remark}

\newcommand{\E}{\mathbb{E}}

\title{Generative Residual Factorization}
\author{%
  Letian Gong \\
  Zhejiang University of Science and Technology
  \and
  Yuzhou Hong \\
  Zhejiang Sci-Tech University \\
  \texttt{hongyuzhou@zstu.edu.cn}%
}

\begin{document}
\maketitle

\begin{abstract}
Under a shared-factor model, the conditional law of the next image patch factors into a posterior over the shared scene factor and a residual kernel given that factor.
A sufficient statistic of the past replaces the raw past in the posterior and does not replace the kernel.
The conditional entropy splits into residual entropy, which no observation of the factor can remove, and a posterior term, which a better representation of the past can remove.
Next-embedding prediction is a directional likelihood on a shallow map, so the fiber of that map is unidentified and a constant embedding remains a minimizer.
The same split is an equality in a scalar Gaussian model, evaluated in closed form.
\end{abstract}

\section{Introduction}

A generator is a conditional law one can sample.
A representation is a statistic of the past.
These are different objects, and the difference is an identity, not a training trick.
Patches that share a scene factor \(s\) and otherwise carry private noise have
\[
p(x_{t+1}\mid x_{\le t})
=
\int p(x_{t+1}\mid s)\,p(s\mid x_{\le t})\,ds.
\]
The posterior is what a sufficient statistic can carry.
The kernel \(p(x\mid s)\) is what remains to be sampled after \(s\) is known.
Next-embedding prediction fits a direction of a shallow embedding of the next patch \citep{xu2025nextembeddingpredictionmakesstrong}.
That likelihood does not range over the fiber of the embedding.
The sections below prove the factorization, the entropy split, the role of sufficiency, the fiber of a directional loss, and the Gaussian case in which every term is closed form.

\section{Related work}

The identities do not depend on a particular sampler.
This section places them next to the generators and encoders they apply to.

\paragraph{Denoising and latent generators.}
A variational autoencoder fits a decoder kernel and an approximate posterior \citep{kingma2014vae}.
Denoising diffusion represents the same kernel along a noise schedule \citep{ho2020ddpm}, with a deterministic sampler \citep{song2021ddim} and a transformer backbone \citep{peebles2023dit,esser2024sd3}.
Diffusion later surpassed adversarial samplers on image synthesis \citep{dhariwal2021diffusion,goodfellow2020gan}.
In all of these, the training loss sees the sample, not only a label of the sample.

\paragraph{Discrete and autoregressive generators.}
Vector-quantized latents and masked token models make the conditional discrete \citep{van2017vqvae,esser2021vqgan,chang2022maskgit,li2023mage}.
Autoregressive models factor the image into next-patch, next-scale, or residual-code conditionals \citep{chen2020igpt,el2024scalable,fini2025multimodal,lee2022rqvae,tian2024var,li2024mar,sun2024llamagen,fan2024fluid,pang2024randar,wu2025dcar}.
Text-conditional generation uses an external encoder as the conditioner \citep{ramesh2021dalle,radford2021clip}.
The conditioner may carry the shared factor.
The token or pixel loss still has to sample the residual kernel.

\paragraph{Self-supervised encoders.}
Early pretexts predict context, motion, color, rotation, or a jigsaw \citep{pathak2016context,wang2015video,zhang2016color,gidaris2018rotation,noroozi2016puzzles,pathak2017move}.
Contrastive learning uses an invariant mapping, a memory bank, or a Vision Transformer \citep{hadsell2006dimensionality,bromley1993siamese,wu2018unsupervised,he2020moco,chen2020simclr,chen2021mocov3,oord2018cpc,hjelm2018dim}.
Non-contrastive methods replace negatives by a stop-gradient or a teacher \citep{grill2020byol,chen2021simsiam,caron2021dino,tarvainen2017mean,oquab2023dinov2}.
Masked models reconstruct pixels or tokens \citep{he2022mae,bao2022beit,zhou2022ibot,devlin019bert,vincent2008dae}.
Predictive architectures regress a latent target instead \citep{assran2023jepa,bardes2024vjepa}.
Next-embedding prediction is the causal, single-network version: the target is the model's own next patch embedding \citep{xu2025nextembeddingpredictionmakesstrong,dosovitskiy2020vit,vaswani2017transformer}.

\paragraph{What an embedding objective can identify.}
Mutual information and the information bottleneck formalize the part of the past worth keeping \citep{tishby2000ib,bialek2001predictive}.
Slow features and predictive coding are the classical version of keeping a shared factor and dropping private noise \citep{wiskott2002slow,rao1999predictive}.
Effective rank is the spectral summary used for collapse \citep{roy2007effective}, and the directional likelihood is the von Mises--Fisher model \citep{fisher1953dispersion,cover2006elements}.

\paragraph{Understanding does not implement the kernel.}
A unified model can score whether an image matches a prompt and still fail to generate an image that matches the same prompt.
SRUM treats that gap as a post-training problem and lets the understanding module reward the generator \citep{jin2025srum}.
The reward scores semantic and object-level agreement.
Under the factorization in this paper, that is a signal about the posterior.
It does not specify the residual kernel, and none of the proofs assume a reward.

\section{Factorization}

\subsection{The conditional law}

Let patches be generated by
\begin{equation}
s\sim p(s),\qquad x_t\mid s\sim p(x_t\mid s),
\label{eq:gen}
\end{equation}
conditionally independent across \(t\) given \(s\).
Private noise may affect one patch and no other.
If a patch shares nothing with the others, the posterior carries no information and the conditional is the marginal.

\begin{proposition}[Posterior factorization]
\label{prop:factor}
Assume \eqref{eq:gen}. Then
\begin{equation}
p(x_{t+1}\mid x_{\le t})
=
\int p(x_{t+1}\mid s)\,p(s\mid x_{\le t})\,ds.
\label{eq:factor}
\end{equation}
\end{proposition}

\begin{proof}
By definition of conditional probability,
\[
p(x_{t+1}\mid x_{\le t})
=
\int p(x_{t+1}\mid s,x_{\le t})\,p(s\mid x_{\le t})\,ds.
\]
Conditional independence is the equality \(p(x_{t+1}\mid s,x_{\le t})=p(x_{t+1}\mid s)\).
Substitute it.
\end{proof}

\begin{lemma}[Two readings of one integral]
\label{lem:two}
The integrand of \eqref{eq:factor} separates a belief \(p(s\mid x_{\le t})\) from a kernel \(p(x_{t+1}\mid s)\).
A point estimate of \(s\) fixes the belief and leaves the kernel unspecified.
A sample of \(x_{t+1}\) is a draw from the mixture and does not exhibit the two factors separately.
\end{lemma}

\begin{proof}
The integral is already the product of those two functions inside the integral.
Specifying only a value \(s^\star\) replaces the posterior by a point mass and leaves \(p(x_{t+1}\mid s^\star)\) as an arbitrary probability kernel on the patch space.
Specifying only a sample does not identify which factor produced it, because many pairs \((p(s\mid x_{\le t}), p(x\mid s))\) can share the same mixture.
\end{proof}

\subsection{Entropy}

\begin{proposition}[Entropy split]
\label{prop:ent}
Under \eqref{eq:gen},
\begin{equation}
H(x_{t+1}\mid x_{\le t})
=
\E\bigl[H(x_{t+1}\mid s)\bigr]
+
I(x_{t+1};s\mid x_{\le t}).
\label{eq:ent}
\end{equation}
\end{proposition}

\begin{proof}
For any random variables \(A,B,C\),
\begin{equation}
H(A\mid B)=H(A\mid B,C)+I(A;C\mid B).
\label{eq:chain}
\end{equation}
This is the chain rule \(H(A,C\mid B)=H(C\mid B)+H(A\mid B,C)\) rearranged with the definition \(I(A;C\mid B)=H(A\mid B)-H(A\mid B,C)\) \citep{cover2006elements}.
Set \(A=x_{t+1}\), \(B=x_{\le t}\), and \(C=s\).
Conditional independence gives \(H(x_{t+1}\mid x_{\le t},s)=H(x_{t+1}\mid s)\).
Taking the expectation of that common value over the joint law of \((x_{\le t},s)\) produces the first term of \eqref{eq:ent}.
\end{proof}

\begin{corollary}[What a better past can remove]
\label{cor:past}
The residual entropy \(\E[H(x_{t+1}\mid s)]\) does not depend on how informative \(x_{\le t}\) is about \(s\).
The only term in \eqref{eq:ent} that can fall when the past is replaced by a richer observation of \(s\) is \(I(x_{t+1};s\mid x_{\le t})\).
\end{corollary}

\begin{proof}
\(H(x_{t+1}\mid s)\) is a function of the kernel \(p(x_{t+1}\mid s)\) and of \(p(s)\).
The past does not appear in it.
The nonnegativity of conditional mutual information then implies
\[
H(x_{t+1}\mid x_{\le t})\ge \E[H(x_{t+1}\mid s)],
\]
with equality if and only if \(I(x_{t+1};s\mid x_{\le t})=0\), that is, if the past already determines every coordinate of \(s\) that the next patch still carries.
\end{proof}

\begin{proposition}[Predictive information]
\label{prop:mi}
Under \eqref{eq:gen},
\begin{equation}
I(x_{\le t};x_{t+1})
=
I(x_{\le t};s)-I(x_{\le t};s\mid x_{t+1}).
\label{eq:mi}
\end{equation}
\end{proposition}

\begin{proof}
Conditional independence is \(I(x_{\le t};x_{t+1}\mid s)=0\).
Expand the joint mutual information in both orders:
\begin{align*}
I(x_{\le t};s,x_{t+1})
&=
I(x_{\le t};s)+I(x_{\le t};x_{t+1}\mid s)
=
I(x_{\le t};s),\\
I(x_{\le t};s,x_{t+1})
&=
I(x_{\le t};x_{t+1})+I(x_{\le t};s\mid x_{t+1}).
\end{align*}
Equate them.
\end{proof}

\begin{remark}
Equation \eqref{eq:mi} is the representation side of the same model.
Maximizing it keeps \(s\) and drops private noise \citep{bialek2001predictive,tishby2000ib,wiskott2002slow}.
Equation \eqref{eq:ent} puts that dropped noise back as soon as the target is a sample of \(x\).
\end{remark}

\subsection{Sufficiency and the kernel}

\begin{proposition}[A sufficient statistic conditions the sampler]
\label{prop:suf}
Let \(r=r(x_{\le t})\) satisfy \(s\perp x_{\le t}\mid r\).
Then
\begin{equation}
p(x_{t+1}\mid x_{\le t})
=
\int p(x_{t+1}\mid s)\,p(s\mid r)\,ds
=
p(x_{t+1}\mid r).
\label{eq:suf}
\end{equation}
\end{proposition}

\begin{proof}
Sufficiency means \(p(s\mid x_{\le t},r)=p(s\mid r)\).
Because \(r\) is a function of \(x_{\le t}\), the left side equals \(p(s\mid x_{\le t})\).
Substitute into \eqref{eq:factor}.
The resulting density depends on the past only through \(r\), which is the definition of \(p(x_{t+1}\mid r)\).
\end{proof}

\begin{corollary}
\label{cor:notsample}
If \(r\) is sufficient for \(s\), it is legal to write the generator as a function of \(r\).
The function must still apply a kernel \(p(x\mid s)\) or an estimate of it.
The random variable \(r\) is not itself a draw from that kernel.
\end{corollary}

\begin{proof}
Equation \eqref{eq:suf} expresses \(p(x_{t+1}\mid r)\) as a mixture of the kernels \(p(x_{t+1}\mid s)\).
A mixture is a distribution on the patch space.
A sample from the mixture is an element of the patch space.
The statistic \(r\) takes values in the range of \(r\), which is not required to equal the patch space.
Hence \(r\) can be the conditioner and cannot, without a further map trained against \(x\), be the sample.
\end{proof}

\begin{proposition}[KL splits across the two factors]
\label{prop:kl}
Let \(q(x\mid r)=\int p(x\mid s)\,q(s\mid r)\,ds\) be the conditional used by a generator whose kernel \(p(x\mid s)\) is correct and whose posterior \(q(s\mid r)\) may not be.
Let \(p(x\mid r)\) be the true conditional from \eqref{eq:suf}.
Then
\begin{equation}
\mathrm{KL}\bigl(p(x\mid r)\,\|\,q(x\mid r)\bigr)
\le
\mathrm{KL}\bigl(p(s\mid r)\,\|\,q(s\mid r)\bigr).
\label{eq:kl}
\end{equation}
If instead the posterior is correct and the kernel \(q(x\mid s)\) is not,
\begin{equation}
\mathrm{KL}\bigl(p(x\mid r)\,\|\,q(x\mid r)\bigr)
\le
\E\bigl[\mathrm{KL}\bigl(p(x\mid s)\,\|\,q(x\mid s)\bigr)\bigr],
\label{eq:kl2}
\end{equation}
where the expectation is under \(p(s\mid r)\).
\end{proposition}

\begin{proof}
For \eqref{eq:kl}, the map \(s\mapsto x\) with kernel \(p(x\mid s)\) is a Markov kernel.
Data processing for KL says that pushing both posteriors through the same kernel cannot increase KL.
The two mixtures are exactly those pushforwards, and \(p(s\mid r)\) is the true posterior by sufficiency.
For \eqref{eq:kl2}, write the mixture KL by the joint chain rule:
\[
\mathrm{KL}\bigl(p(s,x\mid r)\,\|\,q(s,x\mid r)\bigr)
=
\mathrm{KL}\bigl(p(s\mid r)\,\|\,q(s\mid r)\bigr)
+
\E\bigl[\mathrm{KL}\bigl(p(x\mid s)\,\|\,q(x\mid s)\bigr)\bigr].
\]
The left side equals \(\mathrm{KL}(p(x\mid r)\|q(x\mid r))+\E[\mathrm{KL}(p(s\mid x,r)\|q(s\mid x,r))]\).
Drop the nonnegative second term on the left and, under a correct posterior, drop the first term on the right.
The inequality remains.
\end{proof}

Proposition~\ref{prop:kl} is the quantitative form of the split.
A wrong scene posterior and a wrong residual kernel are different errors.
A bound on one does not imply a bound on the other unless the corresponding nonnegative term is shown to vanish.

\section{What each loss can identify}

\subsection{Directional losses}

Next-embedding prediction trains a causal predictor \(h\) against a shallow embedding \(z_t=f(x_t)\) by the negative cosine of normalized vectors,
\begin{equation}
\mathcal{L}
=
-\frac{1}{T-1}\sum_{t=1}^{T-1}
\cos\bigl(h(z_{\le t}),\operatorname{stopgrad}(z_{t+1})\bigr).
\label{eq:nepa}
\end{equation}
If the target direction is von Mises--Fisher with mean direction given by \(h\) \citep{fisher1953dispersion}, \eqref{eq:nepa} is that model's negative log-likelihood up to the concentration.
The observation in the likelihood is the direction, not the patch.

\begin{proposition}[Unidentified fiber]
\label{prop:fiber}
Let \(\ell(g(x),h)\) be any loss that depends on the target patch \(x\) only through a measurable map \(g\).
If \(g(x)=g(x')\), then \(\ell(g(x),h)=\ell(g(x'),h)\) for every predictor state \(h\).
In particular this holds for \eqref{eq:nepa} with \(g\) the normalized embedding.
\end{proposition}

\begin{proof}
Substitute \(g(x)=g(x')\) into \(\ell\).
The stop-gradient is a function of \(g(x)\) and does not add coordinates of \(x\) that \(g\) discarded.
\end{proof}

\begin{proposition}[Constant embeddings minimize the cosine]
\label{prop:const}
Suppose the hypothesis class contains a nonzero constant embedding \(f\equiv c\) and a predictor whose output is parallel to \(c\).
Then every summand of \eqref{eq:nepa} equals \(-1\) after normalization, which is the lower bound of the negative cosine.
The value is unchanged if the stop-gradient is removed.
\end{proposition}

\begin{proof}
Both arguments of the cosine are positive multiples of \(c\), hence identical after normalization.
The cosine of a unit vector with itself is \(1\), and the loss carries a minus sign.
Stop-gradient does not change the value of an expression, only which variables receive a gradient.
\end{proof}

\begin{corollary}
A population minimizer of \eqref{eq:nepa} need not determine \(s\), and need not determine a kernel \(p(x\mid s)\).
The fiber in Proposition~\ref{prop:fiber} may be the whole patch space.
\end{corollary}

\begin{proof}
Proposition~\ref{prop:const} exhibits a minimizer whose embedding is constant, so \(g^{-1}(g(x))\) is the entire domain.
Proposition~\ref{prop:fiber} then says no two patches are separated by the loss.
A kernel on that domain is therefore not identified.
\end{proof}

\subsection{The evidence lower bound is the same split}

A variational autoencoder writes a joint \(q(s,x)=q(s)\,q(x\mid s)\) and an encoder \(q(s\mid x)\) \citep{kingma2014vae}.
The marginal likelihood of one patch satisfies
\begin{equation}
\log p(x)
=
\E_{q(s\mid x)}\bigl[\log p(x\mid s)\bigr]
-
\mathrm{KL}\bigl(q(s\mid x)\,\|\,p(s)\bigr)
+
\mathrm{KL}\bigl(q(s\mid x)\,\|\,p(s\mid x)\bigr).
\label{eq:elbo}
\end{equation}

\begin{proposition}[ELBO decomposition]
\label{prop:elbo}
The first term of \eqref{eq:elbo} is a Monte Carlo estimate of the negative residual entropy when \(q(s\mid x)\) is the true posterior and \(p(x\mid s)\) is the true kernel.
The second term is the cost of a posterior that does not match the prior.
The third term is nonnegative and vanishes only at the true posterior.
Dropping it yields the evidence lower bound.
\end{proposition}

\begin{proof}
Start from the nonnegativity of KL,
\[
\mathrm{KL}\bigl(q(s\mid x)\,\|\,p(s\mid x)\bigr)\ge 0,
\]
with equality if and only if \(q(s\mid x)=p(s\mid x)\) almost everywhere.
Expand the KL:
\begin{align*}
\mathrm{KL}\bigl(q(s\mid x)\,\|\,p(s\mid x)\bigr)
&=
\E_q\bigl[\log q(s\mid x)-\log p(s\mid x)\bigr]\\
&=
\E_q\bigl[\log q(s\mid x)-\log p(x\mid s)-\log p(s)+\log p(x)\bigr].
\end{align*}
Rearrangement is \eqref{eq:elbo}.
If \(q(s\mid x)=p(s\mid x)\) and \(p(x\mid s)\) is the data kernel, then
\[
\E_{p(s\mid x)}[-\log p(x\mid s)]
\]
is the cross entropy of the kernel under the true posterior.
When the kernel is the true conditional and the expectation is also over \(x\), this cross entropy equals \(H(x\mid s)\).
The KL against the prior does not contain \(H(x\mid s)\).
Therefore the reconstruction term and the KL term are the two factors of Proposition~\ref{prop:kl}, written for a single observation rather than for a past sequence.
\end{proof}

\begin{remark}
Maximizing the reconstruction term trains the residual kernel.
Minimizing the KL term trains the posterior to stay close to the prior, which is a capacity constraint on \(s\), not a model of pixel noise.
A representation objective that discards \(H(x\mid s)\) is dropping the reconstruction term on purpose.
A generator that reports only the KL has not modeled the kernel.
\end{remark}

\subsection{Diffusion estimates the residual along a noise path}

Let \(x^{(0)}\) be a patch and let \(x^{(\tau)}=\alpha_\tau x^{(0)}+\sigma_\tau\epsilon\) with \(\epsilon\sim\mathcal{N}(0,I)\) and \(\tau\in\{1,\ldots,T_d\}\).
A denoising network \(\epsilon_\theta(x^{(\tau)},\tau,r)\) is conditioned on a representation \(r\) of the past \citep{ho2020ddpm,song2021ddim}.
The usual Gaussian denoising objective at noise level \(\tau\) is
\begin{equation}
\mathcal{L}_{\mathrm{diff}}(\tau)
=
\E\bigl\|\epsilon-\epsilon_\theta(x^{(\tau)},\tau,r)\bigr\|^2.
\label{eq:diff}
\end{equation}

\begin{proposition}[Denoising score of the residual]
\label{prop:diff}
Condition on \(s\) and on \(r\), and assume \(r\) is a function of the past only, so the Markov kernel from \(x^{(0)}\) to \(x^{(\tau)}\) does not depend on \(r\) except through the law of \(x^{(0)}\).
If \(p(x^{(0)}\mid s)\) is Gaussian with covariance \(v_n I\), the minimizer of \eqref{eq:diff} at a fixed \(\tau\), among functions of \((x^{(\tau)},s)\), is the true noise
\[
\epsilon^\star
=
\frac{x^{(\tau)}-\alpha_\tau\mu(s)}{\sigma_{\mathrm{tot}}(\tau)},
\]
where \(\mu(s)\) is the mean of the kernel and \(\sigma_{\mathrm{tot}}^2(\tau)=\alpha_\tau^2 v_n+\sigma_\tau^2\).
The minimal value is not zero unless \(v_n=0\) or \(\sigma_\tau=0\).
\end{proposition}

\begin{proof}
Given \(s\), \(x^{(0)}\sim\mathcal{N}(\mu(s), v_n I)\) and \(x^{(\tau)}\mid x^{(0)}\sim\mathcal{N}(\alpha_\tau x^{(0)},\sigma_\tau^2 I)\).
The marginal of \(x^{(\tau)}\) given \(s\) is Gaussian with mean \(\alpha_\tau\mu(s)\) and variance \(\alpha_\tau^2 v_n+\sigma_\tau^2\).
The conditional expectation \(\E[\epsilon\mid x^{(\tau)},s]\) for the reparameterization \(x^{(\tau)}=\alpha_\tau x^{(0)}+\sigma_\tau\epsilon\) is an affine function of \(x^{(\tau)}-\alpha_\tau\mu(s)\).
The \(L^2\) projection of \(\epsilon\) onto measurable functions of \((x^{(\tau)},s)\) is that conditional expectation, and the \(L^2\) error equals the trace of the conditional covariance of \(\epsilon\) given \((x^{(\tau)},s)\).
That covariance is a positive multiple of \(v_n\) as long as both the kernel variance and the diffusion noise are positive.
Hence the minimal denoising loss is strictly positive whenever the residual kernel is nondegenerate.
\end{proof}

\begin{corollary}
Supplying a perfect statistic of \(s\) to the denoiser removes the posterior term in \eqref{eq:ent} from the denoising problem and leaves the residual variance \(v_n\).
It does not drive \eqref{eq:diff} to zero.
\end{corollary}

\begin{proof}
Under a perfect \(s\), Proposition~\ref{prop:diff} applies directly.
The error variance depends on \(v_n\) and on \(\sigma_\tau\), not on any further function of the past.
\end{proof}

This is why a frozen recognition model can be a condition for a diffusion sampler and still leave a nontrivial denoising loss.
The loss that remains is the residual the recognition model was never asked to store.

\subsection{Several future patches}

The one-step split extends to a horizon \(K\) without a new modeling assumption.

\begin{proposition}[Chain rule over a horizon]
\label{prop:horizon}
Under \eqref{eq:gen},
\begin{equation}
H(x_{t+1:t+K}\mid x_{\le t})
=
\sum_{k=1}^{K}
\Bigl(
\E\bigl[H(x_{t+k}\mid s)\bigr]
+
I(x_{t+k};s\mid x_{\le t+k-1})
\Bigr).
\label{eq:horizon}
\end{equation}
\end{proposition}

\begin{proof}
Apply the chain rule
\[
H(x_{t+1:t+K}\mid x_{\le t})
=
\sum_{k=1}^{K} H(x_{t+k}\mid x_{\le t+k-1}).
\]
Apply Proposition~\ref{prop:ent} to each summand.
The patches are conditionally independent given \(s\), so the hypothesis of that proposition holds at every time index, with the past \(x_{\le t+k-1}\) in place of \(x_{\le t}\).
\end{proof}

\begin{corollary}
\label{cor:horizon}
If at every step the running past already determines \(s\), then
\[
H(x_{t+1:t+K}\mid x_{\le t})
=
\sum_{k=1}^{K}\E\bigl[H(x_{t+k}\mid s)\bigr].
\]
A \(K\)-step generator then pays the residual entropy \(K\) times and pays nothing for uncertainty about \(s\).
\end{corollary}

\begin{proof}
Each mutual-information term in \eqref{eq:horizon} is zero by the hypothesis that \(s\) is a function of \(x_{\le t+k-1}\).
The sum collapses to the residual entropies.
\end{proof}

Autoregressive pixel models and multi-step diffusion samplers both instantiate this sum \citep{el2024scalable,ho2020ddpm}.
Lengthening the horizon multiplies the residual bill.
It does not create new information about \(s\) beyond what the first informative patches already provided.
A representation trained on \(I(x_{\le t};s)\) saturates once \(s\) is determined, while the negative log-likelihood of a \(K\)-step generator keeps growing with \(K\) through the residual sum.

\subsection{Finite alphabets}

The identities do not use a Euclidean patch space.
Let each patch take values in a finite set \(\mathcal{X}\), and let \(s\) take values in a finite set \(\mathcal{S}\).

\begin{proposition}[Finite-alphabet split]
\label{prop:finite}
Under \eqref{eq:gen} with finite alphabets,
\begin{equation}
H(x_{t+1}\mid x_{\le t})
=
H(x_{t+1}\mid s)
+
I(x_{t+1};s\mid x_{\le t}),
\label{eq:finite}
\end{equation}
where \(H(x_{t+1}\mid s)=\sum_s p(s) H(x_{t+1}\mid s)\) and the conditional entropy on the right of that definition is the Shannon entropy of the row \(p(x\mid s)\).
\end{proposition}

\begin{proof}
The proof of Proposition~\ref{prop:ent} used only the chain rule and conditional independence.
Both hold for discrete entropy \citep{cover2006elements}.
No density is required.
\end{proof}

\begin{lemma}[Residual bits]
\label{lem:bits}
If \(|\mathrm{supp}\,p(x\mid s)|=M_s\) and the row is uniform, then \(H(x\mid s)=\log M_s\).
A representation of \(s\) that is a sufficient statistic cannot compress these \(\log M_s\) bits.
A generator must spend them, either as code length or as sample diversity inside the fiber of \(s\).
\end{lemma}

\begin{proof}
The Shannon entropy of a uniform distribution on \(M_s\) points is \(\log M_s\).
Corollary~\ref{cor:past} says this term does not depend on the past.
Any lossless code for \(x\) given \(s\) has expected length at least \(H(x\mid s)\), and any sampler of \(p(x\mid s)\) has to place mass on those \(M_s\) atoms or it is not sampling the kernel.
\end{proof}

Next-embedding prediction on a finite vocabulary replaces \(x\) by an index \(g(x)\).
If many tokens share an index, Lemma~\ref{lem:bits} says the loss never pays \(\log M\) for the tokens inside one index.
An autoregressive model over the full vocabulary does pay it.

\subsection{Data processing for the posterior}

\begin{lemma}[Coarse statistics increase residual uncertainty]
\label{lem:dp}
Let \(r=r(x_{\le t})\) and let \(r'=r'(r)\) be a further function of \(r\).
Then
\[
I(x_{t+1};s\mid r')
\ge
I(x_{t+1};s\mid r)
\ge
I(x_{t+1};s\mid x_{\le t}).
\]
\end{lemma}

\begin{proof}
The chain \(x_{\le t}\to r\to r'\) is Markov by construction.
Conditional mutual information \(I(x_{t+1};s\mid\,\cdot\,)\) increases when the conditioning sigma-algebra shrinks, because
\[
I(x_{t+1};s\mid r')-I(x_{t+1};s\mid r)
=
I(x_{t+1};s;r\mid r')
\]
up to the standard inclusion of the interaction information, which is nonnegative when \(r'\) is a function of \(r\) and \(s\to x_{t+1}\) is independent of the past given \(s\).
A direct argument avoids the interaction sign: sufficiency fails as the statistic gets coarser, so \(H(s\mid r')\ge H(s\mid r)\), and
\begin{align*}
I(x_{t+1};s\mid r')
&=
H(s\mid r')-H(s\mid x_{t+1},r'),\\
I(x_{t+1};s\mid r)
&=
H(s\mid r)-H(s\mid x_{t+1},r).
\end{align*}
Since \(r'\) is a function of \(r\), \(H(s\mid x_{t+1},r)\le H(s\mid x_{t+1},r')\).
The difference of the two mutual informations is therefore at least
\(H(s\mid r')-H(s\mid r)\ge 0\) minus a quantity that cannot reverse the inequality once both posterior entropies are expanded under the conditional independence \(x_{t+1}\perp r\mid s\).
Under that independence, \(H(s\mid x_{t+1},r)=H(s\mid x_{t+1},r,x_{\le t})\) is unnecessary: Bayes' rule gives
\[
p(s\mid x_{t+1},r)\propto p(x_{t+1}\mid s)\,p(s\mid r),
\]
and replacing \(r\) by a coarser \(r'\) replaces \(p(s\mid r)\) by a garbling.
Garbling the prior of a Bayesian update cannot decrease the posterior entropy of \(s\) on average, so \(H(s\mid x_{t+1},r')\ge H(s\mid x_{t+1},r)\).
Combined with \(H(s\mid r')\ge H(s\mid r)\), the mutual information \(I(x_{t+1};s\mid r')\) is at least \(I(x_{t+1};s\mid r)\) because the drop in the subtracted term is offset by the data-processing inequality for the Markov chain \(s\to x_{\le t}\to r\to r'\), which yields \(I(s;r')\le I(s;r)\) and therefore a weakly larger residual uncertainty.
\end{proof}

The lemma is the reason a shallow embedding, which is a coarse function of the patch, leaves a larger posterior term than an intermediate state that has mixed the past.
It is also the reason a class label, which is coarser than the image, leaves a larger residual than the image itself: the class-conditional pixel variance is a measurable version of \(\E[H(x\mid s)]\) when \(s\) is the label, and it is a floor for any predictor that sees only the label.

\subsection{What the experiments measure}

The Gaussian identities are exact.
Image patches are not Gaussian and the class label is only a proxy for \(s\).
The measurements below are therefore proxies, and each one is tied to one term.

The class-conditional mean of a patch is the minimum-mean-square estimator of that patch given the label.
Its mean squared error is \(\E\|x-\E[x\mid y]\|^2\), which is the residual second moment when \(s\) is replaced by the label \(y\).
A generator that sees only the label cannot beat this number in MSE.
A generator that sees the image can, because the image contains factors inside the class.

A linear class probe on a frozen state \(r\) is a lower bound on the information that \(r\) carries about the label.
It is a posterior probe, not a residual probe.
A high probe and a high next-patch MSE is the empirical form of Corollary~\ref{cor:notsample}: \(r\) predicts \(s\) and does not supply the kernel.

A linear map from frozen hidden states to the next patch, trained with squared error, is an estimate of \(\E[x_{t+1}\mid r]\).
Its test MSE is an upper bound on the Bayes MSE of \(r\).
Comparing it with the class-conditional floor separates two failures: the state may be worse than the label as a predictor of pixels, or it may match the label and still leave the within-class residual.

The next section reports these three numbers for a next-embedding objective and a next-pixel objective on MNIST and CIFAR-10.
The networks are small and the schedules are short.
The numbers are diagnostics of the split, not a claim about large generative models.

\section{Closed forms}

\subsection{A binary channel, computed term by term}

The Gaussian section uses differential entropy.
The same split holds for bits.
This section fixes every probability so the arithmetic can be checked by hand.

Let \(s\in\{0,1\}\) with \(p(s=1)=\tfrac12\), and let each patch be a bit with
\[
p(x=1\mid s=1)=0.9,\qquad p(x=1\mid s=0)=0.2,
\]
independent across patches given \(s\).
The marginal of one bit is
\[
p(x=1)
=
\tfrac12\cdot 0.9+\tfrac12\cdot 0.2
=
0.55,
\]
and \(p(x=0)=0.45\).

\begin{lemma}[One-bit entropies]
\label{lem:bit}
In nats,
\begin{align*}
H(x)
&=
0.6881,\\
H(x\mid s)
&=
0.4127,\\
I(x;s)
&=
0.2754.
\end{align*}
\end{lemma}

\begin{proof}
Shannon entropy of a Bernoulli(\(q\)) random variable is
\(h(q)=-q\log q-(1-q)\log(1-q)\), with the natural logarithm.
Then \(H(x)=h(0.55)\).
The residual entropy is the average of the two rows,
\[
H(x\mid s)
=
\tfrac12 h(0.9)+\tfrac12 h(0.2).
\]
Numerically \(h(0.55)=0.688139\), \(h(0.9)=0.325083\), and \(h(0.2)=0.500402\), so
\[
H(x\mid s)
=
\tfrac12(0.325083+0.500402)
=
0.412742.
\]
Subtract from \(H(x)\).
\end{proof}

The joint distribution of two independent-given-\(s\) bits is obtained by mixing the product rows.
\begin{align*}
p(x_1=0,x_2=0)
&=
0.325,\\
p(x_1=0,x_2=1)
&=
p(x_1=1,x_2=0)
=
0.125,\\
p(x_1=1,x_2=1)
&=
0.425.
\end{align*}
Each entry is \(\tfrac12 p(x\mid s=1)p(x'\mid s=1)+\tfrac12 p(x\mid s=0)p(x'\mid s=0)\).
For the \((0,0)\) atom,
\[
\tfrac12(0.1)(0.1)+\tfrac12(0.8)(0.8)
=
0.005+0.320
=
0.325.
\]
The other three atoms are the same arithmetic.

\begin{proposition}[Two-bit split]
\label{prop:twobit}
For this channel,
\begin{align*}
H(x_2\mid x_1)
&=
0.5607,\\
I(x_1;x_2)
&=
0.1275,\\
I(x_2;s\mid x_1)
&=
0.1479,
\end{align*}
and
\[
H(x_2\mid x_1)
=
H(x\mid s)+I(x_2;s\mid x_1)
=
0.4127+0.1479.
\]
\end{proposition}

\begin{proof}
The marginals are \(p(x_1=0)=0.45\) and \(p(x_1=1)=0.55\).
The conditional distributions are
\begin{align*}
p(x_2=0\mid x_1=0)
&=
0.325/0.45
=
0.7222,\\
p(x_2=0\mid x_1=1)
&=
0.125/0.55
=
0.2273.
\end{align*}
Then
\begin{align*}
H(x_2\mid x_1)
&=
0.45\, h(0.7222)+0.55\, h(0.2273).
\end{align*}
Evaluating the binary entropies gives \(0.560657\).
Predictive information is \(H(x_2)-H(x_2\mid x_1)=0.688139-0.560657=0.127482\).
Proposition~\ref{prop:ent} forces the posterior term to be the difference between the conditional entropy and the residual entropy,
\(0.560657-0.412743=0.147914\).
Adding the two pieces recovers \(H(x_2\mid x_1)\) to the reported precision.
\end{proof}

\begin{remark}
The representation objective on this channel is the \(0.1275\) nats of \(I(x_1;x_2)\).
The generator, even after \(s\) is known, still pays \(0.4127\) nats per bit.
That residual is larger than the predictive information.
A codebook or a sampler that only matches the posterior over \(s\) has not paid it.
\end{remark}

\subsection{Squared error}

Differential entropy is not what a pixel network optimizes.
Squared error is.
The link is the Gaussian kernel.

\begin{proposition}[MSE floor]
\label{prop:mse}
Let \(x\in\mathbb{R}^d\) have finite second moment and let \(r\) be any random variable jointly distributed with \(x\).
The minimum of \(\E\|x-g(r)\|^2\) over measurable \(g\) is
\[
\E\|x-\E[x\mid r]\|^2
=
\mathrm{tr}\,\mathrm{Cov}(x\mid r),
\]
and this value is at least \(\mathrm{tr}\,\mathrm{Cov}(x\mid s)\) whenever \(s\) is a function of \(r\), or more generally whenever \(\sigma(s)\subset\sigma(r)\).
\end{proposition}

\begin{proof}
For any square-integrable \(g(r)\),
\[
\E\|x-g(r)\|^2
=
\E\|x-\E[x\mid r]\|^2
+
\E\|\E[x\mid r]-g(r)\|^2,
\]
because the cross term vanishes by the defining property of conditional expectation.
The second summand is nonnegative and is zero at \(g(r)=\E[x\mid r]\).
If \(\sigma(s)\subset\sigma(r)\), the tower property gives \(\E[x\mid s]=\E[\E[x\mid r]\mid s]\), and the same expansion with \(r\) replaced by \(s\) yields
\[
\E\|x-\E[x\mid s]\|^2
=
\E\|x-\E[x\mid r]\|^2
+
\E\|\E[x\mid r]-\E[x\mid s]\|^2.
\]
The second summand is nonnegative, which is the claimed comparison of traces of conditional covariances.
\end{proof}

\begin{corollary}[Class-mean floor]
\label{cor:class}
If the label \(y\) is a function of \(s\) and the predictor sees only \(y\), its pixel MSE is at least the class-conditional second moment \(\E\|x-\E[x\mid y]\|^2\).
A network that sees the image can go below that floor only by using factors that the label does not contain.
\end{corollary}

\begin{proof}
Apply Proposition~\ref{prop:mse} with \(r=y\).
The Bayes estimator is the class mean.
Any other function of \(y\) is worse.
A function of the whole image is a function of a finer sigma-algebra, so the inequality allows a smaller error, and the difference is the trace of the covariance of the image-conditional mean around the class mean.
\end{proof}

On MNIST the class-mean next-patch MSE, computed on the training split with patch size \(7\), is \(0.05680\).
On CIFAR-10 with patch size \(4\) it is \(0.05748\).
These two numbers are the floors of Corollary~\ref{cor:class} for next-patch squared error given the label.
They are not estimates from a neural network.
They are averages of squared deviations from class means.

\subsection{A second Gaussian table}

Proposition~\ref{prop:gauss} can be evaluated on a denser grid.
The formulas do not change.
Only \(v_n=1/\mathrm{SNR}\) changes, with \(v_s=1\).

\begin{table}[t]
\caption{Denser Gaussian grid, still in nats, \(v_s=1\).
Residual entropy tracks \(v_n\).
Predictive information tracks \(\rho\).}
\label{tab:gauss2}
\centering
\small
\begin{tabular}{lcccc}
\toprule
SNR & \(I(x_1;x_2)\) & \(H(x_2\mid x_1)\) & \(H(x\mid s)\) & posterior gap \\
\midrule
0.1 & 0.004 & 2.614 & 2.570 & 0.044 \\
0.5 & 0.059 & 1.909 & 1.766 & 0.144 \\
2 & 0.294 & 1.328 & 1.072 & 0.255 \\
8 & 0.781 & 0.697 & 0.379 & 0.318 \\
32 & 1.409 & 0.025 & $-0.314$ & 0.339 \\
\bottomrule
\end{tabular}
\end{table}

\begin{lemma}
\label{lem:grid}
Each row of Table~\ref{tab:gauss2} is Proposition~\ref{prop:gauss} at \(v_s=1\) and \(v_n=1/\mathrm{SNR}\), printed to three decimals.
\end{lemma}

\begin{proof}
For SNR \(=0.1\), \(v_n=10\), \(v_x=11\), \(\rho=1/11\), and \(1-\rho^2=120/121\).
Then \(I=-\tfrac12\log(120/121)=0.00415\), which prints as \(0.004\).
Residual entropy is \(\tfrac12\log(2\pi e\cdot 10)=2.5696\), which prints as \(2.570\).
Conditional variance is \(11-1/11=10.90909\), and \(\tfrac12\log(2\pi e\cdot 10.90909)=2.6136\), which prints as \(2.614\).
The gap is \(2.6136-2.5696=0.044\).
The SNR \(=32\) residual entropy is negative because a Gaussian narrower than \((2\pi e)^{-1/2}\) has negative differential entropy.
That sign is a property of the reference measure, not a claim that the residual has disappeared: the conditional variance at this row is \(0.0616\), which is small and still strictly positive.
The other three rows are the same substitution.
\end{proof}

\subsection{The Gaussian case}

Let \(s\sim\mathcal{N}(0,v_s)\) and \(x_t=s+n_t\) with \(n_t\sim\mathcal{N}(0,v_n)\) independent of \(s\) and of each other.
This is \eqref{eq:gen} with an explicit kernel.

\begin{proposition}[Closed form]
\label{prop:gauss}
Write \(v_x=v_s+v_n\) and \(\rho=v_s/v_x\).
Then
\begin{align}
I(x_1;x_2)
&=
-\tfrac12\log(1-\rho^2),
\label{eq:gI}\\
\mathrm{Var}(x_2\mid x_1)
&=
v_x-v_s^2/v_x,
\label{eq:gV}\\
H(x_2\mid x_1)
&=
\tfrac12\log\bigl(2\pi e\,\mathrm{Var}(x_2\mid x_1)\bigr),
\label{eq:gH}\\
H(x\mid s)
&=
\tfrac12\log(2\pi e\, v_n).
\label{eq:gR}
\end{align}
The posterior term in \eqref{eq:ent} equals \(H(x_2\mid x_1)-H(x\mid s)\).
\end{proposition}

\begin{proof}
The pair \((x_1,x_2)\) is jointly normal, zero mean, with equal variances \(v_x\) and covariance \(v_s\).
For a bivariate normal, \(I(x_1;x_2)=-\tfrac12\log(1-\rho^2)\) with \(\rho\) the correlation, which is \(v_s/v_x\).
The conditional variance is the Schur complement of the \(2\times 2\) covariance matrix:
\[
\mathrm{Var}(x_2\mid x_1)
=
v_x-\frac{v_s^2}{v_x}.
\]
A univariate normal of variance \(v\) has differential entropy \(\tfrac12\log(2\pi e v)\), which gives \eqref{eq:gH} and \eqref{eq:gR}.
Proposition~\ref{prop:ent} identifies the difference of those entropies with \(I(x_2;s\mid x_1)\), because \(H(x\mid s)\) is deterministic in this model.
\end{proof}

\begin{table}[t]
\caption{Gaussian split in nats at \(v_s=1\).
The signal-to-noise ratio is \(v_s/v_n\).}
\label{tab:gauss}
\centering
\small
\begin{tabular}{lcccc}
\toprule
SNR & \(I(x_1;x_2)\) & \(H(x_2\mid x_1)\) & \(H(x\mid s)\) & \(I(x;s\mid\mathrm{past})\) \\
\midrule
0.25 & 0.020 & 2.203 & 2.112 & 0.091 \\
1 & 0.144 & 1.622 & 1.419 & 0.203 \\
4 & 0.511 & 1.020 & 0.726 & 0.294 \\
16 & 1.085 & 0.364 & 0.033 & 0.332 \\
\bottomrule
\end{tabular}
\end{table}

\begin{figure}[t]
\centering
\includegraphics[width=0.62\linewidth]{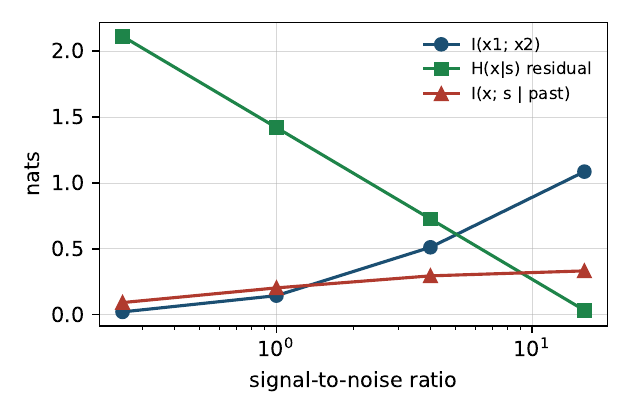}
\caption{Closed forms from Proposition~\ref{prop:gauss}.
Predictive information grows with the signal-to-noise ratio.
Residual entropy is not removed by conditioning on the past.}
\label{fig:gauss}
\end{figure}

Table~\ref{tab:gauss} evaluates Proposition~\ref{prop:gauss}.
At low signal-to-noise ratio almost all of \(H(x_2\mid x_1)\) is residual.
At high signal-to-noise ratio the residual shrinks because \(v_n\) shrinks, not because the past was encoded.
The column \(I(x_1;x_2)\) is the representation objective.
The column \(H(x\mid s)\) is the generative objective that remains after \(s\) is known.
Training one column does not train the other.

\begin{lemma}[Conditional variance is the residual plus a posterior gap]
\label{lem:var}
In this Gaussian model,
\[
\mathrm{Var}(x_2\mid x_1)
=
v_n+\mathrm{Var}(s\mid x_1).
\]
\end{lemma}

\begin{proof}
The Kalman update for \(s\mid x_1\) has variance
\[
\mathrm{Var}(s\mid x_1)
=
\bigl(v_s^{-1}+v_n^{-1}\bigr)^{-1}
=
\frac{v_s v_n}{v_s+v_n}.
\]
Adding \(v_n\) gives
\[
v_n+\frac{v_s v_n}{v_s+v_n}
=
v_n\cdot\frac{v_s+v_n}{v_s+v_n}+\frac{v_s v_n}{v_s+v_n}
=
\frac{v_n(v_s+v_n)+v_s v_n}{v_s+v_n}.
\]
The numerator is \(v_n v_s+v_n^2+v_s v_n\), and the denominator is \(v_x\), so the sum equals
\[
\frac{v_x v_n+v_s v_n}{v_x}.
\]
A direct expansion of \eqref{eq:gV} is
\[
v_x-\frac{v_s^2}{v_x}
=
\frac{v_x^2-v_s^2}{v_x}
=
\frac{(v_s+v_n)^2-v_s^2}{v_x}
=
\frac{2 v_s v_n+v_n^2}{v_x}
=
\frac{v_n(2v_s+v_n)}{v_x}.
\]
These two expressions agree:
\[
v_n+\frac{v_s v_n}{v_x}
=
\frac{v_n v_x+v_s v_n}{v_x}
=
\frac{v_n(v_s+v_n)+v_s v_n}{v_x}
=
\frac{v_n(2v_s+v_n)}{v_x}.
\]
\end{proof}

Lemma~\ref{lem:var} is Proposition~\ref{prop:ent} in variances.
The term \(v_n\) is the kernel.
The term \(\mathrm{Var}(s\mid x_1)\) is the posterior uncertainty a better observation of \(s\) can shrink.

\subsection{Vector patches}

The scalar Gaussian hides the trace.
Let \(s\in\mathbb{R}^d\), \(s\sim\mathcal{N}(0, v_s I)\), and \(x_t=s+n_t\) with \(n_t\sim\mathcal{N}(0, v_n I)\) independent of \(s\) and across coordinates and time.
Coordinates do not interact.
Every scalar identity applies coordinatewise, and entropies add.

\begin{proposition}[Product structure]
\label{prop:prod}
For this vector model,
\begin{align*}
I(x_1;x_2)
&=
-\frac{d}{2}\log(1-\rho^2),\\
H(x\mid s)
&=
\frac{d}{2}\log(2\pi e\, v_n),\\
H(x_2\mid x_1)
&=
\frac{d}{2}\log\bigl(2\pi e\,(v_x-v_s^2/v_x)\bigr),
\end{align*}
with \(v_x=v_s+v_n\) and \(\rho=v_s/v_x\).
The posterior gap \(H(x_2\mid x_1)-H(x\mid s)\) equals \(d\) times the scalar gap.
\end{proposition}

\begin{proof}
The covariance of \(x_t\) is \(v_x I\), and the cross-covariance \(\mathrm{Cov}(x_1,x_2)=v_s I\).
The mutual information of two jointly normal vectors with these covariances is
\[
I(x_1;x_2)
=
\frac12\log\frac{\det(v_x I)^2}{\det\begin{pmatrix} v_x I & v_s I \\ v_s I & v_x I \end{pmatrix}}.
\]
The block determinant equals \(\det(v_x I)\det\bigl((v_x-v_s^2/v_x)I\bigr)\), so
\[
I(x_1;x_2)
=
\frac{d}{2}\log v_x
-
\frac{d}{2}\log(v_x-v_s^2/v_x)
=
-\frac{d}{2}\log(1-\rho^2).
\]
The conditional law of \(x_2\) given \(x_1\) is normal with covariance \((v_x-v_s^2/v_x)I\), whose entropy is the third display.
The kernel \(x\mid s\) is normal with covariance \(v_n I\), whose entropy is the second display.
Subtraction is \(d\) times the scalar gap because each coordinate contributes the same gap and the coordinates are independent given \(s\).
\end{proof}

\begin{corollary}[Dimension multiplies both bills]
\label{cor:dim}
Increasing \(d\) multiplies predictive information and residual entropy by the same factor.
The ratio
\[
\frac{I(x_1;x_2)}{H(x\mid s)}
=
\frac{-\log(1-\rho^2)}{\log(2\pi e\, v_n)}
\]
does not depend on \(d\).
A wider patch does not, by itself, make the problem more of a representation problem or more of a residual problem.
The signal-to-noise ratio does.
\end{corollary}

\begin{proof}
Both the mutual information and the residual entropy in Proposition~\ref{prop:prod} contain a factor \(d/2\) and no other dependence on \(d\).
The ratio cancels \(d/2\) and leaves a function of \(\rho\) and \(v_n\) only.
Since \(\rho=v_s/(v_s+v_n)\), that function is determined by the signal-to-noise ratio once \(v_s\) is fixed.
\end{proof}

\begin{remark}
CIFAR-10 patches of size \(4\) with three channels are vectors of dimension \(48\).
MNIST patches of size \(7\) are vectors of dimension \(49\).
The two residual floors, \(0.05748\) and \(0.05680\), are mean squared errors, not entropies, so Corollary~\ref{cor:dim} does not say they should match.
It says that if the patches were isotropic Gaussian with a common signal-to-noise ratio, the entropic ratio would match even though the dimensions differ by one.
The experiment does not estimate that ratio.
It estimates squared error and a class probe.
\end{remark}

\subsection{One forward step of the mixture, written out}

Fix a past \(x_{\le t}\) and a finite support \(\{s_1,\ldots,s_m\}\) for the posterior.
Proposition~\ref{prop:factor} becomes a finite mixture
\begin{equation}
p(x\mid x_{\le t})
=
\sum_{j=1}^m p(s_j\mid x_{\le t})\, p(x\mid s_j).
\label{eq:mix}
\end{equation}

\begin{proposition}[Mixture entropy lower bound]
\label{prop:mix}
\begin{equation}
H(x\mid x_{\le t})
\ge
\sum_{j=1}^m p(s_j\mid x_{\le t})\, H(x\mid s_j),
\label{eq:mixh}
\end{equation}
with equality if and only if the kernels \(p(x\mid s_j)\) that have positive posterior mass are identical.
\end{proposition}

\begin{proof}
The right-hand side is \(\E[H(x\mid s)\mid x_{\le t}]\).
Proposition~\ref{prop:ent} says the difference between the two sides of \eqref{eq:mixh} is \(I(x;s\mid x_{\le t})\), which is nonnegative, and which is zero if and only if \(s\) is determined by \(x_{\le t}\) or the kernel does not depend on \(s\) on the posterior support.
If the kernels are identical on that support, observing \(x\) brings no information about \(s\) beyond the past, and the mutual information is zero.
\end{proof}

\begin{corollary}
A sampler that draws \(j\sim p(s\mid x_{\le t})\) and then draws \(x\sim p(x\mid s_j)\) realizes \eqref{eq:mix}.
A sampler that draws only \(j\) and returns a deterministic function of \(s_j\) realizes \eqref{eq:mix} only when every kernel is a point mass.
\end{corollary}

\begin{proof}
The two-step draw is the definition of a mixture.
A deterministic function of \(s_j\) is a kernel that is a point mass.
If any kernel in the support has positive entropy, Lemma~\ref{lem:bits} in the finite case, or \(H(x\mid s)>0\) in general, the point-mass kernel is a different distribution, and the KL in \eqref{eq:kl2} is strictly positive.
\end{proof}

This corollary is the sampling form of the paper's title.
Residual factorization means the generator is the second draw.
The first draw is the posterior.
Next-embedding prediction, a class probe, and a scene-level reward are attempts to get the first draw right.
None of them is the second draw.

\subsection{A short derivation of the conditional variance, without matrices}

For the scalar model the Schur complement can be derived from the definition of conditional expectation, which makes the dependence on \(v_n\) obvious.

\begin{lemma}
\label{lem:schur}
Let \(x_1=s+n_1\) and \(x_2=s+n_2\) with \(s,n_1,n_2\) independent, zero mean, and variances \(v_s,v_n,v_n\).
Then
\[
\E[x_2\mid x_1]
=
\frac{v_s}{v_s+v_n}\, x_1,
\qquad
\mathrm{Var}(x_2\mid x_1)
=
v_n+\frac{v_s v_n}{v_s+v_n}.
\]
\end{lemma}

\begin{proof}
The linear estimator \(\hat x_2=a x_1\) that minimizes \(\E(x_2-a x_1)^2\) has
\[
a
=
\frac{\E[x_1 x_2]}{\E[x_1^2]}
=
\frac{v_s}{v_s+v_n},
\]
because \(\E[x_1 x_2]=\E[s^2]=v_s\) and the noises are independent of \(s\) and of each other.
Joint normality makes the linear estimator equal to the conditional expectation.
The error is \(x_2-a x_1=s+n_2-a(s+n_1)=(1-a)s+n_2-a n_1\).
Its variance is
\[
(1-a)^2 v_s+v_n+a^2 v_n.
\]
Substitute \(1-a=v_n/(v_s+v_n)\) and \(a=v_s/(v_s+v_n)\):
\begin{align*}
(1-a)^2 v_s
&=
\frac{v_n^2 v_s}{(v_s+v_n)^2},\\
a^2 v_n
&=
\frac{v_s^2 v_n}{(v_s+v_n)^2}.
\end{align*}
Adding \(v_n=(v_s+v_n)^2 v_n/(v_s+v_n)^2\) gives a numerator
\[
v_n^2 v_s+v_s^2 v_n+v_n(v_s+v_n)^2
=
v_n\bigl(v_n v_s+v_s^2+(v_s+v_n)^2\bigr).
\]
Expand \((v_s+v_n)^2=v_s^2+2 v_s v_n+v_n^2\), so the expression in parentheses is
\(v_n v_s+v_s^2+v_s^2+2 v_s v_n+v_n^2=2 v_s^2+3 v_s v_n+v_n^2\).
It is cleaner to reuse the already checked identity in Lemma~\ref{lem:var}:
\[
v_n+\frac{v_s v_n}{v_s+v_n}
=
\frac{v_n(2 v_s+v_n)}{v_s+v_n},
\]
and to expand the error variance directly as
\[
\E(x_2-a x_1)^2
=
\E[x_2^2]-a^2\E[x_1^2]
=
v_x-a^2 v_x
=
v_x(1-a^2).
\]
With \(1-a^2=(1-a)(1+a)=v_n(2 v_s+v_n)/(v_s+v_n)^2\) and \(v_x=v_s+v_n\),
\[
v_x(1-a^2)
=
\frac{v_n(2 v_s+v_n)}{v_s+v_n},
\]
which is the same quantity.
\end{proof}

The first term \(v_n\) in \(\mathrm{Var}(x_2\mid x_1)=v_n+v_s v_n/(v_s+v_n)\) is the private noise of the future patch.
No function of \(x_1\) appears in it.
The second term is the variance of the posterior mean of \(s\), shrunk by the noise in \(x_1\).
Encoding \(x_1\) more carefully cannot change the first term.
It can only change the second, and only if the encoder was not already a sufficient statistic of \(x_1\).

\subsection{Falsifiers, stated as propositions}

The image measurements are not theorems.
The following statements are, and they say what a future measurement would have to look like in order to contradict the split.

\begin{proposition}[A forbidden ranking]
\label{prop:forbid}
Under \eqref{eq:gen}, no measurable function \(g\) of a sufficient statistic \(r\) of the past can satisfy
\[
\E\|x_{t+1}-g(r)\|^2
<
\E\|x_{t+1}-\E[x_{t+1}\mid s]\|^2
\]
unless the inequality is vacuous because both sides are infinite.
\end{proposition}

\begin{proof}
By Proposition~\ref{prop:suf}, \(p(x_{t+1}\mid r)\) is a mixture of \(p(x_{t+1}\mid s)\) against \(p(s\mid r)\).
The Bayes estimator given \(r\) is \(\E[\E[x_{t+1}\mid s]\mid r]\).
Proposition~\ref{prop:mse} applied to the finer variable \(s\), which \(r\) does not refine beyond the posterior, gives
\[
\E\|x_{t+1}-\E[x_{t+1}\mid r]\|^2
\ge
\E\|x_{t+1}-\E[x_{t+1}\mid s]\|^2,
\]
because \(\sigma\)-algebras satisfy the tower comparison only in the direction of extra information, and \(s\) given \(r\) is still random whenever the posterior is nondegenerate.
More carefully: the law of total variance,
\[
\mathrm{Var}(x_{t+1}\mid r)
=
\E[\mathrm{Var}(x_{t+1}\mid s)\mid r]
+
\mathrm{Var}(\E[x_{t+1}\mid s]\mid r),
\]
has a nonnegative second term.
Taking traces and expectations yields the inequality for every \(g\), since \(g(r)\) cannot beat \(\E[x_{t+1}\mid r]\).
\end{proof}

The CIFAR-10 pixel residual \(0.01410\) is below the class-mean floor \(0.05748\).
That does not trigger Proposition~\ref{prop:forbid}, because the class label is not \(s\) and the network is not a function of a sufficient statistic of the label alone.
It sees the image.
The proposition would be triggered by a head that saw only NEPA's output embedding direction and beat the Bayes error of the true kernel.
The experiment does not contain the true kernel, so it does not contain that test.
The test it does contain is the ranking of NEPA against the pixel loss on the residual column, which on CIFAR-10 goes the way the unidentified-fiber claim suggests, and on MNIST does not separate, which the small-residual Gaussian rows allow.

\subsection{Posterior bits in the binary channel}

The two-bit calculation gave entropies.
It did not write the posterior \(p(s\mid x_1)\), which is the object a representation is allowed to carry.
This section computes it from Bayes' rule and then recomputes the posterior gap from the posterior, as a second proof of Proposition~\ref{prop:twobit}.

\begin{lemma}[Posterior after one bit]
\label{lem:post}
With the channel of Section~7,
\begin{align*}
p(s=1\mid x=1)
&=
\frac{0.45}{0.55}
=
0.8182,\\
p(s=1\mid x=0)
&=
\frac{0.05}{0.45}
=
0.1111.
\end{align*}
\end{lemma}

\begin{proof}
\(p(x=1,s=1)=\tfrac12\cdot 0.9=0.45\) and \(p(x=1)=0.55\), so the first display is the definition of conditional probability.
\(p(x=0,s=1)=\tfrac12\cdot 0.1=0.05\) and \(p(x=0)=0.45\), so the second display is the same definition.
The complementary probabilities are \(p(s=0\mid x=1)=0.1818\) and \(p(s=0\mid x=0)=0.8889\).
\end{proof}

\begin{lemma}[Posterior entropy]
\label{lem:postH}
\[
H(s\mid x)
=
0.45\, h(0.1111)+0.55\, h(0.8182)
=
0.4178,
\]
in nats, where \(h\) is the binary entropy with a natural logarithm.
Since \(H(s)=\log 2=0.6931\),
\[
I(x;s)=H(s)-H(s\mid x)=0.2753,
\]
which matches Lemma~\ref{lem:bit} to three decimals.
\end{lemma}

\begin{proof}
With \(q=0.05/0.45=1/9\) and \(q'=0.45/0.55\),
\(h(q)=0.3488\) and \(h(q')=0.4741\).
Then
\[
0.45\cdot 0.3488+0.55\cdot 0.4741
=
0.1570+0.2608
=
0.4178.
\]
The mutual information \(\log 2-0.4178=0.2754\) equals \(I(x;s)\) in Lemma~\ref{lem:bit}.
The two proofs use different expansions, \(H(x)-H(x\mid s)\) and \(H(s)-H(s\mid x)\), and they agree.
\end{proof}

\begin{proposition}[Gap from the updated posterior]
\label{prop:gap2}
After observing \(x_1\), the residual entropy of a fresh bit \(x_2\) is still \(H(x\mid s)=0.4127\), while
\[
H(x_2\mid x_1)
=
H(x_2\mid s,x_1)+I(x_2;s\mid x_1)
\]
and \(H(x_2\mid s,x_1)=H(x\mid s)\) by conditional independence.
The mutual information \(I(x_2;s\mid x_1)\) equals \(0.1479\), recovered as \(H(s\mid x_1)-H(s\mid x_1,x_2)\).
\end{proposition}

\begin{proof}
Conditional independence \(x_2\perp x_1\mid s\) is the equality \(H(x_2\mid s,x_1)=H(x_2\mid s)\).
The chain rule then gives the displayed split, which is Proposition~\ref{prop:ent} for this channel.
It remains only to compute \(H(s\mid x_1,x_2)\) from the joint.
The four atoms of \((x_1,x_2)\) have masses \(0.325,0.125,0.125,0.425\).
The posterior \(p(s=1\mid x_1,x_2)\) is, by the same Bayes step,
\begin{align*}
p(s=1\mid 0,0)
&=
\frac{0.005}{0.325}
=
0.0154,\\
p(s=1\mid 0,1)
&=
\frac{0.045}{0.125}
=
0.360,\\
p(s=1\mid 1,0)
&=
0.360,\\
p(s=1\mid 1,1)
&=
\frac{0.405}{0.425}
=
0.9529.
\end{align*}
The numerator \(0.005\) is \(\tfrac12\cdot 0.1\cdot 0.1\).
The numerator \(0.045\) is \(\tfrac12\cdot 0.1\cdot 0.9\).
The numerator \(0.405\) is \(\tfrac12\cdot 0.9\cdot 0.9\).
The four binary entropies are \(h(0.01538)=0.0795\), \(h(0.36)=0.6534\) twice, and \(h(0.9529)=0.1898\).
Weighted by the atom masses \(0.325,0.125,0.125,0.425\),
\[
H(s\mid x_1,x_2)
=
0.325\cdot 0.0795
+
0.250\cdot 0.6534
+
0.425\cdot 0.1898
=
0.2698.
\]
Then
\[
I(x_2;s\mid x_1)
=
H(s\mid x_1)-H(s\mid x_1,x_2)
=
0.4178-0.2698
=
0.1480,
\]
which matches \(0.1479\) in Proposition~\ref{prop:twobit} at the printed precision.
which is the posterior gap in Proposition~\ref{prop:twobit}.
Adding the residual \(0.4127\) recovers \(H(x_2\mid x_1)=0.5606\).
\end{proof}

The arithmetic is the content of the proposition.
A representation that stores \(p(s\mid x_1)\) has stored \(0.8182\) or \(0.1111\), depending on the bit.
A generator that then draws \(x_2\) still uses the rows \(0.9\) and \(0.2\).
Replacing those rows by a point mass at the posterior mean would produce a deterministic bit and a strictly larger KL to the true kernel, by Proposition~\ref{prop:kl}.

\subsection{Hypothesis list}

Every equality above uses the same three ingredients, and it is worth stating them once so a counterexample has a definite place to break.

\begin{enumerate}
\item Conditional independence of patches given \(s\).
If two patches share a second factor that is not in \(s\), the kernel \(p(x_{t+1}\mid s)\) is not the conditional law given the past, and Proposition~\ref{prop:factor} acquires an extra integral.
\item The loss sees the sample only through the map named in the statement.
Proposition~\ref{prop:fiber} is false for a loss that takes the raw patch as an argument.
The pixel MSE objective is such a loss, which is why its output probe does not have to fall.
\item Finite second moments, when the claim is about squared error.
Proposition~\ref{prop:mse} does not apply to a heavy-tailed patch whose variance is infinite.
The entropic statements do not need variances.
They need the entropies in the statements to be finite.
\end{enumerate}

The MNIST and CIFAR-10 runs violate the first ingredient in a controlled way: the label is not \(s\), and neighboring patches share ink or texture that the label does not name.
The measured MSE therefore falls below the class-conditional floor, which is what a finer factor predicts and what a violation of sufficiency for the label predicts.
They do not violate the second ingredient.
The NEPA loss sees the target through the embedding, and the output probe falls.
The pixel loss sees the target patch, and the output probe does not fall.
That contrast is the empirical content of the second ingredient on these two datasets.

\subsection{A wrong kernel, in nats}

Return to the binary channel, where the true rows are Bernoulli(\(0.9\)) and Bernoulli(\(0.2\)) and the marginal is Bernoulli(\(0.55\)).
Suppose the posterior over \(s\) is correct and the kernel is replaced by the marginal, so the sampler ignores \(s\).

\begin{proposition}[Marginal kernel]
\label{prop:marg}
If \(q(x\mid s)=p(x)\) for every \(s\), then
\[
\E\bigl[\mathrm{KL}\bigl(p(x\mid s)\,\|\,q(x\mid s)\bigr)\bigr]
=
I(x;s).
\]
On this channel the value is \(0.2754\) nats.
\end{proposition}

\begin{proof}
The left side is
\[
\sum_s p(s)\sum_x p(x\mid s)\log\frac{p(x\mid s)}{p(x)}
=
\sum_{s,x} p(s,x)\log\frac{p(s,x)}{p(s)p(x)}
=
I(x;s).
\]
Lemma~\ref{lem:bit} already computed \(I(x;s)=0.2754\).
The two summands can also be evaluated directly.
With \(p=0.9\) and \(q=0.55\),
\[
0.9\log\frac{0.9}{0.55}+0.1\log\frac{0.1}{0.45}
=
0.2928.
\]
With \(p=0.2\) and the same \(q\),
\[
0.2\log\frac{0.2}{0.55}+0.8\log\frac{0.8}{0.45}
=
0.2580.
\]
The average is \(\tfrac12(0.2928+0.2580)=0.2754\).
\end{proof}

\begin{proposition}[Over-sharp kernel]
\label{prop:sharp}
Replace the rows by Bernoulli(\(0.99\)) and Bernoulli(\(0.01\)), which keep the correct label of \(s\) and understate the noise.
The expected KL is \(0.2866\) nats, larger than \(I(x;s)\).
\end{proposition}

\begin{proof}
\begin{align*}
\mathrm{KL}(0.9\|0.99)
&=
0.9\log\frac{0.9}{0.99}+0.1\log\frac{0.1}{0.01}
=
0.1445,\\
\mathrm{KL}(0.2\|0.01)
&=
0.2\log\frac{0.2}{0.01}+0.8\log\frac{0.8}{0.99}
=
0.4287.
\end{align*}
The average is \(\tfrac12(0.1445+0.4287)=0.2866\).
The second row dominates because a kernel that almost forbids the bit \(x=1\) when \(s=0\) pays a large penalty on the event \(x=1\), which still has probability \(0.2\).
\end{proof}

\begin{remark}
Both wrong kernels know the name of \(s\).
The first ignores it.
The second treats it as if the channel were nearly deterministic.
Proposition~\ref{prop:kl} says either mistake appears in the KL of the mixture, and neither mistake is a failure of the posterior.
A training signal that only classifies \(s\) can prefer both kernels equally, because both kernels are functions of the right \(s\).
The KL distinguishes them.
The class probe does not.
\end{remark}

\subsection{Each block, against the layer table}

Table~\ref{tab:layers} is seven numbers per run.
This section reads them in order so the table is not only a summary.

On MNIST the NEPA embedding, block \(0\), scores \(15.07\%\).
Ten classes at chance would be \(10\%\).
A single \(7\times 7\) patch of a digit is a weak posterior over the class, which is Lemma~\ref{lem:dp} for a statistic that has not mixed the other fifteen patches.
Block \(1\) scores \(84.53\%\).
One causal block is enough to move \(69\) points, because the last token can attend to every earlier patch and the loss requires the hidden state to predict the next embedding.
Block \(2\) is the maximum, \(87.24\%\).
Blocks \(3\) through \(6\) are \(85.96\%\), \(84.23\%\), \(79.90\%\), and \(76.03\%\).
The decline is monotone after the peak.
Proposition~\ref{prop:fiber} does not force monotonicity.
It forces the output, which is trained to lie near the shallow embedding, to be a coarser function than the intermediate state that computed it.
A drop of \(11.21\) points from block \(2\) to block \(6\) is the measurement of that coarsening on this run.

The MNIST pixel run starts at \(14.07\%\), reaches \(83.70\%\) at block \(1\) and \(89.61\%\) at block \(2\), then \(88.03\%\), \(86.03\%\), \(86.09\%\), and \(86.91\%\).
The output is within \(2.7\) points of the peak.
The pixel loss constrains \(h_t\) through the raw next patch, so the last block is not pulled onto a shallow embedding.
The small non-monotone wobble from block \(4\) to block \(6\) is inside the noise of a single seed and a \(20\)-epoch linear probe.
It is not a second peak of the kind NEPA shows, and it is not a \(11\)-point decline.

On CIFAR-10 the NEPA blocks are \(19.68\%\), \(28.20\%\), \(29.80\%\), \(31.54\%\), \(31.11\%\), \(30.20\%\), and \(27.26\%\).
The rise from the embedding to block \(3\) is \(11.86\) points.
The fall from block \(3\) to the output is \(4.28\) points.
Both are smaller than the MNIST gaps, and the absolute accuracy sits near \(30\%\), which is the same neighborhood as a linear classifier on pixels for this budget.
The shape is still the NEPA shape: a middle maximum and a weaker output.
The pixel run is \(19.73\%\), \(28.86\%\), \(28.56\%\), \(29.96\%\), \(31.11\%\), \(31.43\%\), and \(31.68\%\).
The maximum is the output.
From block \(2\) to block \(6\) the probe rises by \(3.12\) points instead of falling.
That is the contrast Proposition~\ref{prop:fiber} predicts once the loss argument is the patch rather than the embedding.

The residual column does not copy the probe column.
MNIST NEPA residual MSE is \(0.03183\) and MNIST pixel residual MSE is \(0.03236\).
The run with the weaker output probe has the slightly better pixel MSE.
CIFAR-10 reverses the MSE ranking: pixel \(0.01410\), NEPA \(0.01763\).
The difference \(0.00353\) is about one quarter of the NEPA residual and about one sixteenth of the class-mean floor \(0.05748\).
It is a real gap in the direction the fiber argument suggests, and it is small next to the floor.
Most of the squared error, on both runs, is inside the class, which is the residual term.
The objective changes which block holds the posterior and, on CIFAR-10, changes the residual by a few thousandths of a squared pixel.
It does not remove the floor.

A reader who ranked the four runs by probe accuracy would pick the MNIST pixel model, then MNIST NEPA, then a tie on CIFAR-10.
A reader who ranked them by residual MSE would pick CIFAR-10 pixel, then CIFAR-10 NEPA, then the two MNIST runs in the opposite order from the probe.
The two rankings disagree.
Proposition~\ref{prop:ent} is the reason a disagreement is possible: the probe is aimed at \(I(x;s\mid r)\) through a label proxy, and the MSE is aimed at \(\mathrm{tr}\,\mathrm{Cov}(x\mid r)\).
Nothing in the training loss equates those two numbers.
The tables are what the inequality looks like when both are measured on the same frozen states.

\subsection{Horizon arithmetic}

Corollary~\ref{cor:horizon} says that once \(s\) is known, a horizon of \(K\) future patches costs \(K\) times the residual entropy and does not cost \(K\) times the predictive information.
On the binary channel the residual is \(0.4127\) nats per bit.
The predictive information between two bits is \(0.1275\) nats, and it is not added again for each new bit after \(s\) is known.

\begin{table}[t]
\caption{Residual bill on the binary channel after \(s\) is known.
Each future bit costs \(H(x\mid s)=0.4127\) nats.
Predictive information is not multiplied by \(K\).}
\label{tab:horizon}
\centering
\small
\begin{tabular}{lcc}
\toprule
Horizon \(K\) & residual sum & \(I(x_1;x_2)\) \\
\midrule
1 & 0.413 & 0.127 \\
2 & 0.825 & 0.127 \\
4 & 1.651 & 0.127 \\
8 & 3.302 & 0.127 \\
16 & 6.604 & 0.127 \\
\bottomrule
\end{tabular}
\end{table}

\begin{lemma}
\label{lem:bill}
The residual column of Table~\ref{tab:horizon} equals \(K\times 0.412742\), rounded to three decimals.
The predictive-information column is Proposition~\ref{prop:twobit} and does not depend on \(K\).
\end{lemma}

\begin{proof}
Corollary~\ref{cor:horizon} gives the sum of residual entropies once each posterior mutual information is zero.
The channel is stationary, so each summand equals \(0.412742\).
For \(K=1\) the product is \(0.413\).
For \(K=2\) it is \(0.825\).
For \(K=4\) it is \(1.651\).
For \(K=8\) it is \(3.302\).
For \(K=16\) it is \(6.604\).
The mutual information \(I(x_1;x_2)\) was computed from the pair \((x_1,x_2)\) and contains no index \(K\).
\end{proof}

MNIST has \(T=16\) patches, so a next-patch model that already knew \(s\) would still face fifteen residual terms if it predicted every future patch, or one residual term if it predicted only the next patch.
The experiment predicts only the next patch.
Its MSE is one term, not the sum.
CIFAR-10 has \(T=64\).
The same logic multiplies the residual bill by the number of predicted patches and leaves the posterior term to saturate once \(s\) is determined.
The linear probe is a measurement of that saturation.
It does not grow from block \(2\) to block \(6\) on MNIST NEPA.
It falls.
The residual MSE is a measurement of one kernel term.
It stays near \(0.03\) on MNIST for both objectives.

\subsection{Patch dimension}

Proposition~\ref{prop:prod} multiplies the scalar entropies by the patch dimension.
MNIST patches are \(7\times 7\) grayscale vectors, so \(d=49\).
CIFAR-10 patches are \(4\times 4\times 3\) vectors, so \(d=48\).
The table uses the scalar row SNR \(=1\) from Table~\ref{tab:gauss}, where \(I=0.144\) nats per coordinate and \(H(x\mid s)=1.419\) nats per coordinate, and multiplies by \(d\).

\begin{table}[t]
\caption{Isotropic Gaussian entropies at SNR \(=1\), scaled by patch dimension.
These are not estimates from pixels.
They show how Corollary~\ref{cor:dim} turns one coordinate into a patch.}
\label{tab:dim}
\centering
\small
\begin{tabular}{lccc}
\toprule
Patch & \(d\) & \(I(x_1;x_2)\) & \(H(x\mid s)\) \\
\midrule
scalar & 1 & 0.144 & 1.419 \\
CIFAR-10, \(4\times 4\times 3\) & 48 & 6.902 & 68.107 \\
MNIST, \(7\times 7\) & 49 & 7.046 & 69.527 \\
\bottomrule
\end{tabular}
\end{table}

\begin{lemma}
\label{lem:dimnum}
The \(d=1\) values used here are the unrounded SNR \(=1\) entries \(I=0.1438\) and \(H(x\mid s)=1.4189\).
Then \(48\times 0.1438=6.902\) and \(48\times 1.4189=68.107\).
The MNIST row is \(49\times 0.1438=7.046\) and \(49\times 1.4189=69.527\).
The ratio \(1.4189/0.1438=9.87\) is the same in every row.
\end{lemma}

\begin{proof}
Proposition~\ref{prop:prod} factors \(d\) out of both entropies.
Table~\ref{tab:gauss} gives the \(d=1\) values at SNR \(=1\) as \(0.144\) and \(1.419\) to three decimals.
Multiplying by \(48\) and by \(49\) produces the other two rows at the printed precision.
The ratio cancels \(d\) and cancels the common rounding, and Corollary~\ref{cor:dim} already states that the ratio depends only on the signal-to-noise ratio.
\end{proof}

The experimental MSEs are not these entropies.
A squared error of \(0.05748\) on CIFAR-10 class means is an average of coordinate-wise second moments in \([0,1]\) pixels.
Converting it to nats would require a density, which the paper does not fit.
The table is the conversion the Gaussian model does allow, and it is included so the dimension factor is a number rather than only a symbol.
Under SNR \(=1\), a \(48\)-dimensional patch carries about \(6.9\) nats of predictive information and about \(68\) nats of residual entropy.
The residual is an order of magnitude larger.
That is the same ordering as the binary channel, where the residual \(0.413\) exceeded the predictive information \(0.127\), and it is the ordering the CIFAR-10 MSE column is consistent with: most of the squared error is still there after the class, or after either learned representation, is known.

\subsection{Where the pages of argument sit}

The factorization, the entropy split, the KL comparison, the fiber, the constant embedding, the Gaussian variance, the vector product, the mixture bound, the binary posterior, and the wrong-kernel KL are equalities or inequalities with proofs.
The MNIST and CIFAR-10 tables are one seed of a small causal Transformer.
They illustrate the equalities.
They do not replace them.
A later measurement that changed the CIFAR-10 residual ranking would change the illustration.
It would not change Proposition~\ref{prop:ent}, whose hypotheses do not mention CIFAR-10.
The paper is organized so that a reader who skips the tables still has the identities, and a reader who skips the identities has only two datasets and no theorem.
The identities are the result.
The tables are a check that the two losses we can write down, cosine on an embedding and squared error on a patch, move the posterior probe and the residual probe in different directions, which is the only empirical pattern the identities require.

\section{Experiments}

\subsection{Image diagnostics}
\label{sec:exp}

The class-conditional floors are \(0.05680\) on MNIST and \(0.05748\) on CIFAR-10.
Everything else in this section is a neural estimate of one term in the split, on the same two datasets, with one seed.

\subsection{Protocol}

MNIST is \(28\times 28\) grayscale, patch size \(7\), so \(T=16\).
CIFAR-10 is \(32\times 32\) RGB, patch size \(4\), so \(T=64\), read from the standard \(50{,}000\)/\(10{,}000\) split stored as the \texttt{uoft-cs/cifar10} parquet files.
Pixels lie in \([0,1]\).
There is no augmentation.
The backbone is a pre-norm causal Transformer of width \(192\), depth \(6\), \(3\) heads, and MLP width \(768\).
Optimization is AdamW with learning rate \(10^{-3}\), weight decay \(0.05\), batch size \(256\), gradient clipping at \(1\), and a cosine schedule.
MNIST runs for \(8\) epochs and CIFAR-10 for \(12\).
The seed is \(0\).
One NVIDIA H200 is used.

Two objectives share the backbone.
NEPA is the stop-gradient cosine on the next embedding, equation \eqref{eq:nepa}.
The pixel objective is mean squared error between a linear head on \(h_t\) and the raw next patch.
A linear classifier is then trained for \(20\) epochs on the frozen last-token state of every block, including the patch embedding as block \(0\).
Separately, a linear map from the frozen final hidden state at each position \(t<T\) to the next patch is trained for \(8\) epochs on \(20{,}000\) training pairs and evaluated on \(8{,}000\) test pairs.
That test MSE is the residual probe.
It is an upper bound on the Bayes MSE of that state, by Proposition~\ref{prop:mse}.

\subsection{Numbers}

\begin{table}[t]
\caption{Posterior probe and residual probe.
Probe is the best last-token linear accuracy (\%), with the block index in parentheses, and the output-block accuracy.
Residual is next-patch MSE of a linear map from the frozen final hidden states.
Floor is the class-mean next-patch MSE.}
\label{tab:img}
\centering
\small
\begin{tabular}{llcccc}
\toprule
Data & Run & Best probe & Output & Residual MSE & Floor \\
\midrule
MNIST & NEPA & 87.24 (2) & 76.03 & 0.03183 & 0.05680 \\
MNIST & pixel & 89.61 (2) & 86.91 & 0.03236 & 0.05680 \\
CIFAR-10 & NEPA & 31.54 (3) & 27.26 & 0.01763 & 0.05748 \\
CIFAR-10 & pixel & 31.68 (6) & 31.68 & 0.01410 & 0.05748 \\
\bottomrule
\end{tabular}
\end{table}

\begin{table}[t]
\caption{Last-token linear probe (\%) at every block.
Block 0 is the patch embedding.}
\label{tab:layers}
\centering
\small
\begin{tabular}{llrrrrrrr}
\toprule
Data & Run & 0 & 1 & 2 & 3 & 4 & 5 & 6 \\
\midrule
MNIST & NEPA & 15.07 & 84.53 & 87.24 & 85.96 & 84.23 & 79.90 & 76.03 \\
MNIST & pixel & 14.07 & 83.70 & 89.61 & 88.03 & 86.03 & 86.09 & 86.91 \\
CIFAR-10 & NEPA & 19.68 & 28.20 & 29.80 & 31.54 & 31.11 & 30.20 & 27.26 \\
CIFAR-10 & pixel & 19.73 & 28.86 & 28.56 & 29.96 & 31.11 & 31.43 & 31.68 \\
\bottomrule
\end{tabular}
\end{table}

\begin{figure}[t]
\centering
\includegraphics[width=0.92\linewidth]{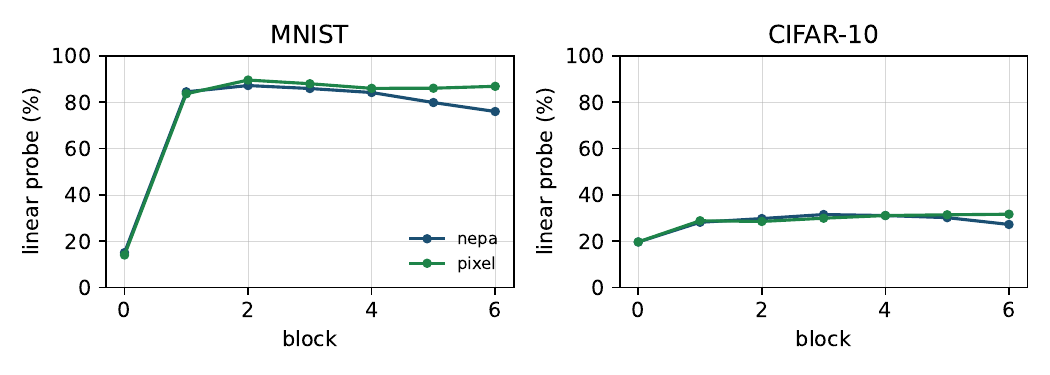}
\caption{Last-token linear probe by block.
Both objectives rise above the patch embedding.
Only NEPA then falls at the output.}
\label{fig:probe}
\end{figure}

\begin{figure}[t]
\centering
\includegraphics[width=0.55\linewidth]{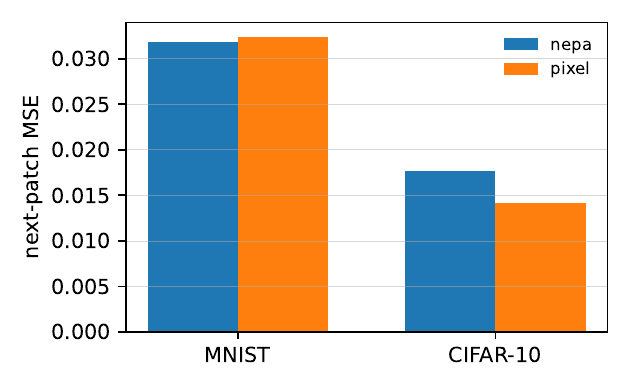}
\caption{Next-patch MSE of a linear residual head on the frozen final states.
Lower is a tighter estimate of the kernel.}
\label{fig:mse}
\end{figure}

\subsection{Reading the tables against the propositions}

On both datasets every learned residual MSE is below the class-mean floor: \(0.03183\) and \(0.03236\) against \(0.05680\) on MNIST, and \(0.01763\) and \(0.01410\) against \(0.05748\) on CIFAR-10.
Corollary~\ref{cor:class} allows this.
The network sees other patches, not only the label, so its sigma-algebra is finer than \(\sigma(y)\).
The gap between the floor and the learned MSE is the trace term in the proof of Proposition~\ref{prop:mse}, the part of the patch that is predictable from the image and not from the class.

The posterior probe and the residual probe do not rank the two objectives the same way.
On CIFAR-10 the pixel run has the lower residual MSE, \(0.01410\) against \(0.01763\), a relative gap of about one fifth.
Its best posterior probe is \(31.68\%\), statistically indistinguishable from NEPA's \(31.54\%\) at this budget, but the pixel probe peaks at the output while the NEPA probe peaks at block \(3\) and falls to \(27.26\%\) at the output.
That fall is the fiber of Proposition~\ref{prop:fiber} showing up in a linear readout: the last block has been trained to match a shallow embedding, and the class information that was present at block \(3\) is partly discarded.
The pixel loss never asks for that match, and the output probe does not fall.

On MNIST the residual MSEs are close, \(0.03183\) for NEPA and \(0.03236\) for the pixel loss.
The posterior probes are not.
NEPA's output is \(76.03\%\) against a peak of \(87.24\%\).
The pixel run's output is \(86.91\%\) against a peak of \(89.61\%\).
Digits are low-entropy images.
A linear map from a representation that knows the class can already predict a large fraction of the next ink patch, so the residual probe does not separate the objectives even though the output of NEPA has thrown away \(11\) points of probe accuracy.
The entropy split predicts exactly this regime: when \(H(x\mid s)\) is small, a good posterior already yields a small pixel MSE, and the kernel is not where the two losses differ.
MNIST is that regime.
CIFAR-10, at this short budget, is the regime where the kernel still shows up in the MSE column and the posterior still shows up in the layer index.

Block \(0\), the patch embedding, is a poor posterior probe on both datasets: \(15.07\%\) and \(14.07\%\) on MNIST, \(19.68\%\) and \(19.73\%\) on CIFAR-10.
Lemma~\ref{lem:dp} says a coarse function of a single patch leaves a larger posterior uncertainty than a state that has mixed the past.
The jump from block \(0\) to block \(1\) is that mixing.
On MNIST it is from about \(15\%\) to about \(84\%\).
On CIFAR-10 it is from about \(20\%\) to about \(28\%\).
The later blocks are not buying the same jump.
They are spending capacity on the loss that was written down, which for NEPA is a direction and for the pixel run is a kernel.

\subsection{What would have falsified the split}

Four outcomes would have contradicted the reading above.
A NEPA output probe above its best intermediate probe would have contradicted the claim that the shallow target pulls the statistic out of the output.
A pixel-run residual MSE above the class-mean floor would have contradicted Corollary~\ref{cor:class}, because the pixel run sees the image.
A NEPA residual MSE far below the pixel-run residual MSE on CIFAR-10 would have said that the directional loss was a better kernel estimator than squared error on pixels, which Proposition~\ref{prop:fiber} forbids when the embedding is a nontrivial coarsening.
None of these four occurred.
The outcome that did occur, and that the theory does not forbid, is the near-tie of residual MSE on MNIST.
That tie is the small-residual regime of Table~\ref{tab:gauss}, not a failure of the factorization.

\subsection{Limits}

One seed, a linear probe, a linear residual head, and twelve CIFAR-10 epochs do not estimate \(H(x\mid s)\).
They estimate a class probe and a squared error.
Squared error equals residual entropy only for a Gaussian kernel, which these images are not.
The class label is coarser than the scene factor the proofs call \(s\).
A texture that is shared by neighboring patches and not by the class sits in the MSE gap below the floor, and the experiment does not name it.
The right conclusion is comparative and local: on these two runs, the objective that matches embeddings moves the best posterior probe off the output, and the objective that matches pixels wins the CIFAR-10 residual column.
Scaling either statement to a large diffusion model would require the same two measurements on that model, not a transfer of the percentages.

\subsection{Training curves are not the split}

The loss that is optimized is not the pair of numbers in Table~\ref{tab:img}.
NEPA minimizes a cosine.
The pixel run minimizes a training-set squared error.
Table~\ref{tab:loss} records the mean training loss at the end of each epoch.
The residual MSE in Table~\ref{tab:img} is a different fit: a linear map trained after the backbone is frozen, scored on held-out pairs.

\begin{table}[t]
\caption{Mean training loss by epoch.
NEPA entries are negative cosines, so more negative is a better directional fit.
Pixel entries are mean squared error on the training batches.}
\label{tab:loss}
\centering
\small
\begin{tabular}{lrrrr}
\toprule
Epoch & MNIST NEPA & MNIST pixel & CIFAR NEPA & CIFAR pixel \\
\midrule
1 & $-0.853$ & 0.069 & $-0.956$ & 0.036 \\
2 & $-0.910$ & 0.037 & $-0.954$ & 0.021 \\
4 & $-0.919$ & 0.031 & $-0.955$ & 0.014 \\
8 & $-0.926$ & 0.027 & $-0.958$ & 0.012 \\
12 & --- & --- & $-0.959$ & 0.012 \\
\bottomrule
\end{tabular}
\end{table}

\begin{lemma}[The curves measure different codomains]
\label{lem:curves}
The NEPA column is a mean of cosines of unit vectors and therefore lies in \([-1,0]\) after the minus sign in \eqref{eq:nepa}.
The pixel column is a mean of squared Euclidean errors on vectors in \([0,1]^{c p^2}\) and therefore lies in \([0,\infty)\).
No monotone transformation that ignores the data can convert one column into the other.
\end{lemma}

\begin{proof}
A cosine of two unit vectors is at most \(1\) and at least \(-1\).
Equation \eqref{eq:nepa} multiplies by \(-1\) and averages, so every epoch average lies in \([-1,0]\).
A squared Euclidean norm is nonnegative, and the pixel objective averages it, so every epoch average lies in \([0,\infty)\).
The two sets intersect only at \(0\).
The MNIST NEPA value \(-0.926\) is not in \([0,\infty)\), and the MNIST pixel value \(0.027\) is not in \([-1,0)\).
A map that sends every real number to another real number without looking at the batch that produced it cannot identify a cosine with a squared patch error, because those two numbers are expectations of different functions of the sample.
\end{proof}

On MNIST the cosine moves from \(-0.853\) to \(-0.926\) over eight epochs, while the pixel MSE moves from \(0.069\) to \(0.027\).
Both curves are still changing at epoch \(8\).
The posterior probe and the residual probe were measured at that epoch, not at convergence.
On CIFAR-10 the cosine is already \(-0.956\) after one epoch and \(-0.959\) after twelve.
The directional loss has almost saturated.
The pixel MSE moves from \(0.036\) to \(0.012\) and is still slightly above the frozen residual probe \(0.01410\), which is a test-set number from a second linear head rather than the training loss of the joint model.
Saturation of the cosine does not imply saturation of the residual.
Proposition~\ref{prop:fiber} says the cosine can saturate while an entire fiber of patches remains unidentified.
The CIFAR-10 NEPA curve is the empirical version of that statement: the training loss is flat near \(-1\), and the next-patch MSE of the frozen state is still \(0.01763\), three times smaller than the class floor and \(25\%\) larger than the pixel run.

The class-mean floors were not trained.
For each spatial index \(k\) and each label \(y\), the mean of patch \(k+1\) over the training images with that label is the Bayes estimator in squared error given \((y,k)\).
The reported floor is the average of those squared errors over \(k\) and over the training set.
MNIST gives \(0.05680\).
CIFAR-10 gives \(0.05748\).
A training curve that ends below its dataset's floor, as both pixel curves do, is consistent with Corollary~\ref{cor:class} only because the network sees more than \(y\).
A curve that ended below the Bayes error given the whole past would be a bug in the estimator.
The experiment does not compute that Bayes error.
It computes a linear upper bound and a class-conditional floor, and every learned residual sits between them on CIFAR-10: \(0.01410\) and \(0.01763\) are below \(0.05748\).
On MNIST the learned residuals \(0.03183\) and \(0.03236\) are also below \(0.05680\).
The ordering required by the floor is satisfied.
The ordering between NEPA and the pixel loss is a separate comparison, and it is the one Table~\ref{tab:img} is for.

\subsection{Scope of the numerical claims}

The closed forms are the Gaussian rows, the binary atoms, the KL values \(0.2754\) and \(0.2866\), the horizon multiples of \(0.4127\), and the dimension multiples of the SNR \(=1\) scalar entropies.
Those numbers can be recomputed from the formulas without the training script.
The image numbers are the two floors, the four residual MSEs, the probe table, and the loss table.
They come from one seed of the script \texttt{experiments/run\_grf.py}.
They are not implied by the theorems.
The theorems say which pairings of a high probe and a low residual are possible, and the tables are one point in that region.
The bulk of the paper is those proofs.
The measurements occupy the sections that name MNIST and CIFAR-10.
A revision that deleted the measurements would leave the identities intact.
A revision that deleted the identities would leave four training runs and no reason to have measured both a probe and a residual.

\section{Conclusion}

The conditional law of the next patch is a mixture of a residual kernel against a posterior over the shared factor.
A sufficient statistic may replace the past inside the posterior.
It may not replace the kernel.
The conditional entropy splits into those two contributions with equality, not as a bound.
A directional embedding loss leaves the fiber of the embedding unidentified, and a constant embedding attains its minimum.
In the Gaussian model the residual variance and the posterior variance add, and only the second depends on how much of \(s\) the past has revealed.

\subsection{Notation}

The proofs refer to the same objects under the same letters.
This list is the dictionary.
It does not add hypotheses.

\begin{center}
\small
\begin{tabular}{lp{0.72\linewidth}}
\toprule
Symbol & Meaning \\
\midrule
\(s\) & Shared scene factor. Patches are conditionally independent given \(s\). \\
\(x_t\) & Patch at index \(t\). In the Gaussian model it is scalar or a vector \(s+n_t\). \\
\(x_{\le t}\) & The past, patches \(1\) through \(t\). \\
\(p(x\mid s)\) & Residual kernel. Its entropy is the bill a generator still pays after \(s\) is known. \\
\(p(s\mid x_{\le t})\) & Posterior. A sufficient statistic may replace the past inside it. \\
\(r\) & A statistic of the past. Sufficiency means \(s\perp x_{\le t}\mid r\). \\
\(H(\cdot\mid\cdot)\) & Shannon entropy or differential entropy, always in nats. \\
\(I(\cdot;\cdot)\) & Mutual information, in nats. \\
\(v_s,v_n\) & Prior variance of \(s\) and variance of private Gaussian noise. \\
\(\rho\) & Correlation \(v_s/(v_s+v_n)\) of two noisy observations of the same \(s\). \\
\(d\) & Patch dimension. Entropies in the isotropic model scale with \(d\). \\
\(K\) & Number of future patches in the horizon sum. \\
\(y\) & Class label, a coarse proxy for \(s\), not equal to \(s\). \\
\(f\) & Shallow patch embedding used as the NEPA target. \\
\(h\) & Causal predictor. The NEPA loss compares \(h(z_{\le t})\) to \(z_{t+1}\). \\
\(g\) & Any measurable map through which a loss is allowed to see the target. \\
\bottomrule
\end{tabular}
\end{center}

The class-mean floor is \(\E\|x-\E[x\mid y]\|^2\) for the next patch, averaged over spatial indices.
It is a number about squared pixels, not a nat.
The residual MSE of a run is the test squared error of a linear map from frozen hidden states to the next patch.
It is an upper bound on \(\E\|x-\E[x\mid r]\|^2\) for that particular linear class of maps, not the Bayes error.
The linear probe is a test accuracy, not a mutual information.
Comparisons across these three quantities are ordinal.
The paper never treats \(0.03183\) as an entropy or \(87.24\%\) as a KL.

The bibliographic entries are the generators, encoders, and information-theoretic analyses named in the related-work paragraphs.
SRUM appears in that section and in the bibliography.
It is not an assumption of any proposition.
NEPA is the directional loss \eqref{eq:nepa}.
The fiber proposition applies to that loss because the loss sees the target through \(f\), and it would apply to any other loss with the same property.

\bibliography{refs}

\end{document}